\documentclass[lettersize,journal]{IEEEtran}
\usepackage{amsmath,amsfonts}
\usepackage{algorithmic}
\usepackage{algorithm}
\usepackage{array}
\usepackage[caption=false,font=normalsize,labelfont=sf,textfont=sf]{subfig}
\usepackage{textcomp}
\usepackage{stfloats}
\usepackage{url}
\usepackage{verbatim}
\usepackage{graphicx}
\usepackage{cite}
\usepackage{multirow}
\usepackage{diagbox} 
\usepackage{booktabs}
\usepackage{makecell}
\usepackage{xcolor}
\usepackage{color,varioref}
\newcommand{\red}[1]{\begin{color}{black}#1\end{color}}

\usepackage[colorlinks,linkcolor=blue]{hyperref}

\usepackage{amsthm}   
\usepackage{amssymb}
\newtheorem{theorem}{Theorem}
\newtheorem{assumption}{Assumption}
\newtheorem{definition}{Definition}
\newtheorem{lemma}{Lemma}
\newtheorem{remark}{Remark}

\begin{document}

\title{Decoupled Structure-Feature Alignment via Alternating Optimization for Graph Learning}



\author{
Chengcheng Yan, Feifei Zhao, Dai Zhu,
Wei Liu,
and Qingsong Wang
%
%
\thanks{Corresponding author: Qingsong Wang.}
\IEEEcompsocitemizethanks{
\IEEEcompsocthanksitem Chengcheng Yan and Wei Liu are with the College of Artificial Intelligence, Shaoxing Institute of Technology, Shaoxing, 312000, China. (e‑mail: ycc956176796@gmail.com, weiliu@zju.edu.cn).
\IEEEcompsocthanksitem Feifei Zhao is with College of Science, Shihezi University, Shihezi, 832003, China. (email: zhaofeifei@shzu.edu.cn)
\IEEEcompsocthanksitem Dai Zhu is with Hunan Shaofeng Institute for Applied Mathematics, Xiangtan, 411102, China. (e‑mail: partout@gmail.com).
\IEEEcompsocthanksitem Qingsong Wang is with School of Mathematics and Computational Science, Xiangtan University, National Center for Applied Mathematics in Hunan, Xiangtan, 411105, China. (email: nothing2wang@hotmail.com). 

}

}



\maketitle

\begin{abstract}

Conventional Graph Neural Networks (GNNs) couple feature transformation and neighborhood aggregation, which often renders them vulnerable to topological noise and heterophilous connections. To decouple this dependency, we present a constrained two-view learning framework for robust graph learning, which aligns structure-aware GNN embeddings with a structure-free feature prior. Specifically, the proposed Decoupled Structure-Feature Alternating Learning  (DSAL) framework trains an independent anchor network using a self-supervised reconstruction objective to capture the intrinsic semantic information contained in node attributes. Within DSAL, to effectively integrate this prior, we design a channel-split adaptive gated (CSAG) layer. This architecture employs a gating mechanism to balance global spectral smoothing and local spatial representation dynamically. Furthermore, the model is optimized via a cyclic alternating procedure, which mitigates representation drift caused by mutual interference in standard joint optimization schemes. Experiments on diverse homophilous and heterophilous datasets suggest that our proposed approach provides improved node classification accuracy while maintaining robustness to structural perturbations compared to standard message-passing architectures.
\end{abstract}


\section{Introduction}
\IEEEPARstart{G}{raph} neural networks (GNNs) \cite{corso2024graph} learn node representations by combining node attributes with information propagated through the graph structure. Given the topological structure $A$ and the node feature matrix $X$ of a graph, the conventional training objective of GNNs is typically formulated as the following unconstrained optimization problem:
\begin{equation}
\label{eq:original}
    \min_{\theta} \mathcal{L}(\Psi(A, X ; \theta), Y) ,
\end{equation}
where $\Psi (\cdot; \theta)$ denotes the neural network model for node-level prediction and $\theta = \{W_1, W_2, \dots, W_L\}$ denotes all trainable parameters of the GNN. Symbol $\mathcal{L}$ denotes the loss function, and $Y$ denotes the labels of nodes. Given an $L$-layer GNN, for problem \eqref{eq:original}, $\Psi$ denotes the composition of $L$ message-passing layers defined as follows:
\begin{equation*}
\begin{aligned}
    \Psi(A, X; \theta) = \mathcal{H}^L_\theta,  &\quad  \mathcal{H}^l_\theta = \sigma(A \mathcal{H}^{l-1}_\theta W_l) , \\  \mathcal{H}^0_\theta = X   , & \quad  l=1,\dots,L   ,
\end{aligned}
\end{equation*}
where $\sigma$ is an activation function, $H^l_\theta$ denotes the hidden representation at the $l$-th layer, and $W_l$ are the weight parameters of GNN. By adopting the aggregation operator $\sigma(A \mathcal{H}^{l-1}_\theta W_l)$, the model performs spatial smoothing over node representations in the non-Euclidean graph domain. This coupling allows GNNs to jointly transform node features and aggregate neighborhood information to learn stronger feature representations.

\begin{figure*}[h]
    \centering
    \includegraphics[width=1.0\linewidth]{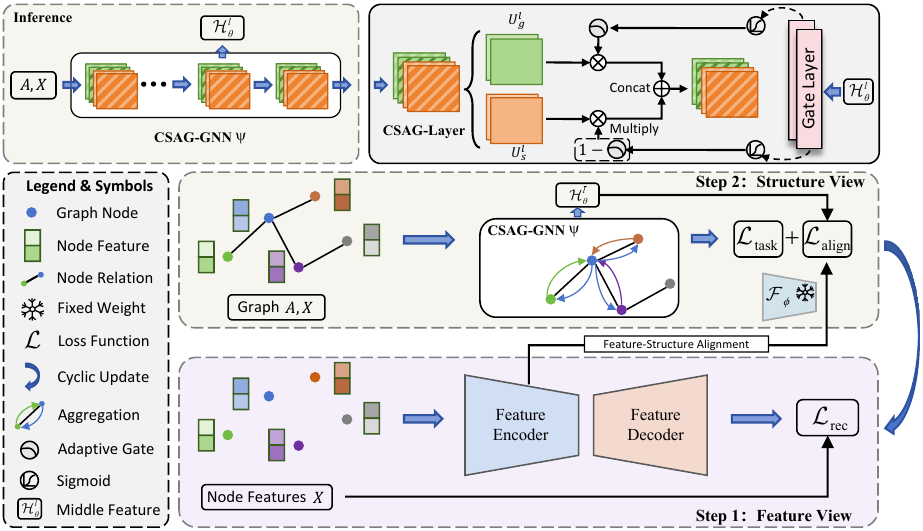}
    \caption{An overview of the proposed DSAL framework. The model decouples graph learning into a Feature View (Step 1) and a Structure View (Step 2) optimized via a cyclic alternating scheme. The feature encoder distills a prior $\mathcal{F}_\phi$ without structural information, which acts as an alignment anchor ($\mathcal{L}_{\text{align}}$) for the intermediate embeddings of DSAL. The embedded CSAG-Layer dynamically balances dual-branch aggregations using an adaptive gating mechanism.}
    \label{fig:csag-gnn}
\end{figure*}

However, the above procedure implicitly assumes that the observed topological structure $A$ is reliable and semantically consistent with the node labels \cite{platonov2023critical}. In practice, real-world graphs are often noisy, incomplete, and may even contain spurious or misleading edges \cite{zhou2023opengsl}. Moreover, heterophilous graphs violate many of the homophily assumptions underlying message-passing architectures. Therefore, structural aggregation may propagate erroneous signals across neighborhoods, thereby degrading representation quality and yielding unreliable predictions \cite{yang2021attributes, yan2022two}. As a result, the standard training paradigm fails to fully exploit the inherent discriminative information embedded in the original feature $X$. Several studies attempt to improve robustness through graph structure learning \cite{jin2020graph, zheng2020robust}, data augmentation \cite{zhao2021data}, or regularization techniques \cite{feng2020graph, chen2023agnn, yang2021rethinking}. Nevertheless, these methods still optimize representations from the graph topology structure and remain inherently dependent on the given topology $A$. When the topology is unreliable or corrupted, the learned embedding vectors may still drift \cite{wu2022nodeformer}.

To mitigate this problem, we instead use the node attributes as a topology-independent source of information. Specifically, we learn a feature representation from node features $X$ alone and use it as an anchor for the structural representation. This motivates us to construct an anchor model $\mathcal{F}_\phi$ from the independent observation dimension (i.e., feature view) that learns representations independent of the graph structure via an auxiliary loss function $\mathcal{L}_{\text{rec}}$, which is a self-supervised reconstruction loss for the encoder-decoder procedure. In this manner, $\mathcal{F}_\phi$ serves as a feature semantic prior. It captures the intrinsic attribute of the graph nodes and provides a stable reference for aligning the representations from the graph structure.

Therefore, we couple the two views through an explicit consistency constraint: \red{the alignment formulation explicitly couples two complementary information sources, ensuring that the semantics distilled from $X$ are injected into the task representation, while the task supervision prevents the feature prior from drifting away from the discriminative structure.} In this work, we reformulate the training of GNN as a constrained optimization problem by introducing the consistency constraint:
\begin{equation}
\begin{aligned}
\label{eq:main_problem}
\min_{\theta, \phi} \quad & \mathcal{L}_{\text{task}}( \Psi( A, X; \theta), Y)   +  \mathcal{L}_{\text{rec}}(\mathcal{F}_\phi (X)) \\
\text{s.t.} \quad & \mathcal{H}_{\theta}^{\bar{l}} = \mathcal{F}_\phi   , \quad  \bar{l} = \lceil L/2 \rceil,
\end{aligned}
\end{equation}
where $\lceil \cdot \rceil$ denotes the ceil operator. The consistency constraint encourages the middle-structured representation $\mathcal{H}_\theta^{\bar{l}}$ produced by the GNN to be aligned with the feature representation $\mathcal{F}_\phi$, which can be interpreted as imposing a regularization on the structural manifold induced by the feature manifold. \red{We align the feature prior with an intermediate representation rather than the final-layer output. The intermediate representation is less directly tied to the final classification layer and therefore provides the representation used for the consistency constraint.}

\red{
An intuitive way to solve problem \eqref{eq:main_problem} is to incorporate the consistency constraint condition through a penalty term and perform joint end-to-end optimization with both the task loss and the reconstruction loss. However, under joint backpropagation, gradients from $\mathcal{L}_{\text{task}}$ are propagated to $\phi$ through $\mathcal{L}_{\text{align}}$. When the topological structure $A$ contains substantial noise, these structural inaccuracies may be transferred into $\phi$, which can weaken the effectiveness of the learned feature prior from the feature view. Instead, we separate the two updates: the anchor network is trained only with the reconstruction loss $\mathcal{L}_{rec}$, after which the GNN is updated using the task loss and the alignment term, with the anchor held fixed. To this end, this can be interpreted as solving the following asymmetric two-view optimization problem:}
\begin{equation}
\begin{aligned}
\label{eq:mainloss}
\left\{
\begin{aligned}
\phi^{k+1}
&\approx \arg\min_{\phi} \mathcal{L}_{\text{rec}}\big(\mathcal{F}_\phi(X)\big),
\\
\theta^{k+1}
&\approx \arg\min_{\theta} 
\Big\{
\mathcal{L}_{\text{task}}\big(\Psi(A,X;\theta),Y\big)
+
\mathcal{L}_{\text{align}}\big(
\mathcal{H}_\theta^{\bar l},
\mathcal{F}_{\phi^{k+1}}
\big)
\Big\},
\end{aligned}
\right.
\end{aligned}
\end{equation}
where the feature align loss function $\mathcal{L}_{\text{align}} =  D (\mathcal{H}_\theta^{\bar{l}}, \mathcal{F}_{\phi^{k+1}} )$, here $D(x, y) = \frac{\rho}{2} \| x - y \|^2$ is a quadratic penalty function measuring the discrepancy between two representations. In practice, $\phi^{k+1}$ is not solved exactly. Instead, we maintain an iterative estimate $\phi^k$, which is updated via a few steps of stochastic gradient descent on $\mathcal{L}_{\text{rec}}$ and serves as an online approximation to $\phi^{k+1}$. The subsequent structure-view update then aligns $\mathcal{H}_\theta^{\bar{l}}$ with the current anchor model $\mathcal{F}_{\phi^{k+1}}$. Specifically, during the update procedure, the anchor representation $\mathcal{F}_{\phi^{k+1}}$ is treated as a fixed reference, and gradients are not propagated back to $\phi$.

Formally, under this inexact optimization scheme, \eqref{eq:mainloss} gives rise to the proposed Decoupled Structure-Feature Alternating Learning (DSAL) framework. Specifically, as shown in Figure \ref{fig:csag-gnn}, we alternately update the feature-view parameters $\phi$ to strengthen a semantic prior without structure information. The structure-view parameters $\theta$ using a Channel-Split Adaptive Gated (CSAG) layer to optimize the task objective while aligning to the current prior, to obtain more balanced task prediction results.

In this paper, our main contributions are summarized as follows:
\begin{itemize}
    \item We introduce a two-view perspective for node prediction that explicitly bridges a structure encoder and a structure-free feature prior. Specifically, we introduce an anchor (feature-view) model that learns semantic representations from node attributes alone and aligns this prior information with intermediate structural embeddings produced by the GNN, resulting in a constrained optimization method for robust graph learning.
    \item To optimize the asymmetric two-view optimization problem \eqref{eq:mainloss}, we use a cyclic alternating update to train the feature and structure views separately. During the structure update, the feature representation is fixed, which prevents task gradients from being propagated into the feature-view encoder. In addition, we provide a theoretical analysis of the proposed algorithm in the \hyperref[appendix_list]{Section IV}.
    \item We evaluate the method on six node classification benchmarks covering both homophily and heterophily, and further examine component contributions, topology perturbations, update strategies, hyperparameter sensitivity, and computational overhead.
\end{itemize}

\section{Related Work}
\label{sec:related}

\textbf{Graph Neural Networks and Topological Robustness.} GNNs are widely used for modeling data with non-Euclidean structures. Their core mechanism is message passing, where node representations are updated by aggregating information from neighboring nodes. In methods such as GCN \cite{kipf2017semi} and GAT \cite{velickovic2018graph}, the learned representations are closely coupled with the input adjacency matrix A, making model performance sensitive to the quality of graph topology.

In practice, graph structure may contain noise arising from missing edges, spurious connections, or adversarial perturbations. Such imperfections can affect information propagation across layers and lead to over-smoothing \cite{wu2019simplifying}. To address these challenges, existing methods can be broadly categorized into Graph Structure Learning (GSL) \cite{jin2020graph} and regularization approaches designed to improve robustness \cite{chen2023agnn, yang2021rethinking, liu2024robust, ning2025summary}. GSL methods aim to refine or reconstruct the adjacency matrix during training, often by learning pairwise similarities or optimizing structural consistency with task objectives. Regularization approaches, on the other hand, introduce additional constraints or data augmentation strategies to improve robustness under structural perturbations.

Despite these efforts, most approaches rely on joint optimization of model parameters and graph structure within a coupled framework. This coupling makes them sensitive to the initial graph, which may be unreliable. In cases of severely corrupted topology or high heterophilous graphs \cite{zhu2020beyond}, optimization that relies heavily on the graph topology may lead to suboptimal representations.

\textbf{Two-View Learning and Feature–Structure Alignment on Graphs.} To overcome the limitations of relying solely on topology, two-view learning and graph contrastive learning have been widely adopted. These methods typically treat the feature space and the structure space as parallel sources of information. For example, the Deep Graph Infomax (DGI) \cite{velickovic2018deep} method enhances model generalization by maximizing the feature information between node representations that rely on graph structure and representations derived solely from node features through a Multi-Layer Perceptron (MLP). Furthermore, related work in graph knowledge distillation includes GLNN \cite{zhang2021graphless}, which transfers structural knowledge from a GNN teacher to an MLP student model, enabling inference without explicit graph structure. 

However, existing two-view alignment strategies predominantly employ a joint optimization paradigm with multiple objectives, formulated as $\mathcal{L}_{\text{total}} = \mathcal{L}_{\text{task}} + \lambda \mathcal{L}_{\text{align}}$. In scenarios where the graph topology is unreliable, this single-level joint optimization is prone to feature drift. Due to the lack of independent constraints on the feature view, clean node feature representations can become contaminated by erroneous gradients stemming from noisy topology, causing the two views to align within an incorrect semantic space \cite{han2023alternately}.

\textbf{Alternating Optimization in Neural Network Models.} Alternating optimization is commonly used when different parameter blocks are associated with different objectives or constraints. In deep learning \cite{lecun2015deep}, alternating optimization is frequently employed to solve complex objective functions involving multiple sets of non-convex parameters, such as in meta-learning \cite{hospedales2021meta} and generative adversarial networks (GANs) \cite{goodfellow2020generative}. Unlike traditional stochastic gradient descent (SGD) \cite{robbins1951stochastic}, which directly minimizes an unconstrained joint loss, alternating optimization \cite{zeng2019global} minimizes one set of variables while holding another fixed. In these methods, one parameter block is optimized while the other is held fixed, which separates the corresponding update steps. While a limited number of studies in the field of graph learning \cite{xia2021graph} have explored alternating optimization to update graph structures and network parameters, few have applied this technique to the consistent alignment of feature and structural representations.

The approximate solution framework proposed in \eqref{eq:mainloss} leverages alternating optimization to isolate gradient interference between $\mathcal{F}_\phi$ and $\mathcal{H}_\theta$. This block-coordinate update mechanism \cite{beck2013convergence, cai2023cyclic} ensures that the feature prior model, while independently capturing the attribute manifold, serves as a stable anchor for structural view optimization, thereby preventing mutual drift during the representation alignment process.

\section{Methodology}
\label{sec:method}

\subsection{Notations}

Let $\mathcal{G} = (\mathcal{V}, \mathcal{E})$ denote a graph, where $\mathcal{V} = \{v_1, v_2, \cdots, v_n\}$ is the set of $n$ nodes and $\mathcal{E}$ is the set of edges. The topological structure of $\mathcal{G}$ is represented by an adjacency matrix $A \in \mathbb{R}^{n \times n}$. Specifically, $A_{ij} > 0$ indicates the presence of an edge between nodes $v_i$ and $v_j$, while $A_{ij} = 0$ otherwise. We define the input node feature matrix as $X \in \mathbb{R}^{n \times d}$, where $d$ is the dimension of the node attributes. Furthermore, let $l \in \{1, \cdots, L\}$ denote the layer index. The symbols $H_g^l \in \mathbb{R}^{n \times d_l}$ and $H_s^l \in \mathbb{R}^{n \times d_l}$ denote the latent representations on the $l$-th layer, extracted by the Graph Convolutional Network (GCN) branch and the GraphSAGE branch, respectively. Additionally, the symbol $\| \cdot \|$ is defined as the $\ell_2$ norm $\| \cdot \|_2$ for vector inputs, and as the Frobenius norm $\| \cdot \|_F$ for matrix inputs.

\subsection{Overview of DSAL Framework}

In the above optimization framework, the original GNN is responsible for fitting the target labels, while the model $\mathcal{F}$ serves to guide the GNN using representations derived solely from the node features $X$. The overall update procedure is summarized in Algorithm~\ref{algo:csag-code}.

\begin{algorithm}[!ht]
\caption{DSAL Framework}
\begin{algorithmic}[1]
\label{algo:csag-code}
\STATE $\mathbf{Input:~}$ Adjacency matrix $A$, Node features $X$, Node label $Y$, Max iteration $K$, Hyperparameter $\rho$, and Latent representations $\mathcal{H}_\theta^{\bar{l}}$.
\STATE $\mathbf{Initialize:~}$ Anchor network weights $\phi^0$, CSAG-GNN weights $\theta^0$.
\FOR{ $k = 0, \cdots, K-1$ }
    \STATE \textbf{Step 1: Update weights $\phi$ of anchor network}
    \STATE \quad Compute the reconstruction loss to refine the feature prior: 
    \[ \mathcal{L}_{\text{rec}}(\mathcal{F}_{\phi^k} (X)) = \| \mathcal{F} ( X ; \phi^k,\cdot )  - X \|^2 \]
    \vspace{-0.35cm}
    \STATE \quad Update $\phi$ by taking stochastic gradient descent steps:
    \[ \phi^{k+1} \leftarrow \phi^{k} - \eta \nabla_{\phi} \mathcal{L}_{\text{rec}}(\mathcal{F}_{\phi^k} (X)) \]
    \vspace{-0.35cm}
    \STATE \textbf{Step 2: Update weights $\theta$ of CSAG-GNN}
    \STATE \quad Freeze the updated anchor network $\mathcal{F}_{\phi^{k+1}}$.
    \STATE \quad Compute the feature alignment constraint:
    \[ \mathcal{L}_{\text{align}} (\mathcal{H}_\theta^{\bar{l}} , \mathcal{F}_{\phi^{k+1}})   = D (\mathcal{H}_{\theta}^{\bar{l}} , \mathcal{F}_{\phi^{k+1}}) = \frac{\rho}{2} \| \mathcal{H}_{\theta}^{\bar{l}} - \mathcal{F}_{\phi^{k+1}} \|^2 \]
    \vspace{-0.35cm}
    \STATE \quad Update $\theta$ by minimizing loss via gradient descent:
    \[ \theta^{k+1} \leftarrow \arg\min_\theta \mathcal{L}_{\text{task}}( \Psi(  A, X ;\theta) , Y) + D (\mathcal{H}_\theta^{\bar{l}} , \mathcal{F}_{\phi^{k+1}} ) \]
\ENDFOR
\STATE $\mathbf{Output:~}$ Model weights $\theta^K, \phi^K$.
\end{algorithmic}
\end{algorithm}

\subsubsection{Feature View: Anchor Network $\mathcal{F}_\phi$}

In our analysis, to generate a reliable anchor feature representation for $\mathcal{F}_\phi$, inspired by the idea of the AutoEncoder method  \cite{hinton2006reducing}, we designed an AutoEncoder network $\mathcal{F}$ to obtain the final node feature. Therefore, the corresponding reconstruction loss has
\begin{equation}
    \min_{\phi, \vartheta } \quad \| \mathcal{F} ( X ;\phi,\vartheta )  - X \|^2   ,
\end{equation}
where $\mathcal{F}(\phi) := \mathcal{F}_{\phi}$ denotes the encoder and $\mathcal{F}(\vartheta) := \mathcal{F}_{\vartheta}$ denotes the decoder in the AutoEncoder network $\mathcal{F}$. Under this formulation, from the feature perspective, the reconstruction loss encourages the encoder $\mathcal{F}_{\phi}$ to extract the intrinsic representation of the original graph nodes, thereby providing effective guidance for downstream task training.

As the primary objective of an AutoEncoder is to reconstruct the input data, the reconstruction procedure consists of two stages: encoding and decoding. From Step 1 in Figure \ref{fig:csag-gnn}, the encoder learns to extract informative latent representations from the input features, whereas the decoder reconstructs the input from these representations. In our framework, we are primarily interested in the representations learned by the encoder. Therefore, when training the anchor network, we explicitly define the training objective of the encoder $\mathcal{F}_{\phi}$ as follows:
\begin{equation}
    \mathcal{L}_{\text{rec}}(\mathcal{F}_\phi (X)) = \| \mathcal{F} ( X ; \phi,\cdot )  - X \|^2    .
\end{equation}

\subsubsection{Structure View: Adaptive Dual-Branch Aggregation for CSAG-GNN}

\textbf{(a) Dual-Branch Aggregation.}  To simultaneously capture global spectral dependencies and local spatial context, we use two aggregation branches to capture different neighborhood patterns. The GCN branch provides the smoothing behavior of graph convolution, while the GraphSAGE branch uses sampled neighborhood aggregation. To this end, we first consider a dual-branch architecture. In this framework, given the input middle feature $\mathcal{H}^l_{\theta}$, the output feature $\mathcal{H}^{l+1}_\theta \in \mathbb{R}^{n \times d_{l+1}}$ is concatenated from the two aforementioned features. As shown in Figure \ref{fig:csag-gnn} and \eqref{split-ori}, two branch outputs with complementary operation:
\[
U_g^l = \text{GCN}(A, \mathcal{H}^l_{\theta}) \in \mathbb{R}^{n \times d_g}, \\  \; U_s^l = \text{SAGE}(A, \mathcal{H}^l_{\theta}) \in \mathbb{R}^{n \times d_s},
\]
where $d_g + d_s = d_{l+1}$ is controlled by the channel allocation ratio $\gamma$. 

Therefore, by joining the output feature $U_g^l$ from GCN layer and the output feature $U_s^l$ from GraphSAGE layer at $l$-th layer, the output feature $\mathcal{H}^{l+1}_{\theta}$ can be formulated
\begin{equation}
\label{split-ori}
    \mathcal{H}^{l+1}_{\theta} = \left[U^{l}_{g} \parallel U^{l}_{s}\right], \quad d_{l+1} = d_g = d_s ,
\end{equation}
where the symbol $\parallel$ denotes the channel-wise concatenation operation and the two branch outputs $U_g^l \in \mathbb{R}^{n \times d_g}$ and $U_s^l \in \mathbb{R}^{n \times d_s}$ respectively. Symbol $n$ denotes the total number of nodes, and $d$ denotes the number of channels. 

Nevertheless, according to \eqref{split-ori}, if the inputs to the network layers are kept identical, concatenating the outputs of two layers along the channel dimension results in a doubled feature dimensionality, thereby increasing the parameter cost. To alleviate this issue, we let the output channels of the two layers are proportionally allocated, such that the output dimensionality matches that of the original layer, while the number of parameters lies between the two individual layers. Based on this analysis, \eqref{split-ori} can reformulated as
\begin{equation}
\label{split-final1}
    \mathcal{H}^{l+1}_{\theta} = \left[U^{l}_{g} \parallel U^{l}_{s}\right], \quad d_{g} = \lfloor \gamma \cdot d_l \rfloor, \quad d_{s} = d_l - d_{g}, 
\end{equation}
where $\lfloor \cdot \rfloor$ denotes the floor operation and the channel allocation ratio $\gamma \in (0, 1)$. In this style, compared to enforcing a single aggregation mechanism, we decompose the output representation space into two subspaces, allowing the model to allocate specific channel capacities to different topological views.

\textbf{(b) Dual Aggregation with Adaptive Gate Layer.}  
Although the dual-branch style \eqref{split-final1} successfully generates complementary representations (spectral smoothing via GCN and spatial preservation via GraphSAGE), a simple concatenation provides no explicit mechanism to adaptively regulate their contributions. For example, in graphs with varying levels of homophily, the preferred balance between smoothing and local discrimination may vary across nodes: nodes within homophilous communities may benefit from stronger smoothing, whereas boundary nodes or nodes embedded in heterophilous neighborhoods may require more discriminative aggregation to reduce neighbor contamination \cite{pei2024multi}.

To resolve the above conflict, we need a special layer controller that dynamically modulates the two branches based on the current node representation, enabling soft routing between global smoothing and local discrimination. For this purpose, we introduce a lightweight gate $\alpha^l \in (0,1)$ computed from the layer input $\mathcal{H}^l_{\theta}$. Therefore, the gate layer can be formulated as:
\[
\alpha^l = \sigma(W_\alpha \mathcal{H}^l_{\theta} + b_\alpha) \in \mathbb{R}^{n \times 1},
\]
where $\sigma(\cdot)$ denotes the sigmoid function and $W_{\alpha}$ is a learnable weight vector of the gate layer, $b_\alpha$ means the bias. The gate adds only a small number of trainable parameters relative to the graph convolution layers. It can adaptively modulates the contributions of the two branches before fusion:
\[
\tilde{U}_g^l = \alpha^l \odot U_g^l, \qquad \tilde{U}_s^l = (1 - \alpha^l) \odot U_s^l ,
\]
where $\odot$ denotes the Hadamard product, implemented via the broadcast mechanism. The gate assigns a weight to each node for the two branch outputs before concatenation. We empirically validate the necessity of the gating mechanism in the ablation study.

Therefore, the output $\mathcal{H}^{l+1}_{\theta}$ of the CSAG-Layer is merged by channel-wise concatenation. Thus, for \eqref{split-final1}, which can be written as:
\begin{equation}
\label{split-final2}
\mathcal{H}^{l+1}_{\theta} = \left[\tilde{U}_g^l \parallel \tilde{U}_s^l\right].
\end{equation}

The value of $\alpha$ determines the relative contribution of the GCN and GraphSAGE branches: larger values increase the weight of the GCN output, while smaller values favor the GraphSAGE output. Remarkably, this dynamic adaptability is achieved with negligible parameter overhead (only $O(d_{l})$).

\subsection{Self-Consistent Cyclic Learning}
As an inexact solution to the optimization problem \eqref{eq:mainloss}, we adopt the following approximate alternating update scheme:

\textbf{Step 1: Update weights $\phi$ of anchor network.} We update the weight parameters $\phi$ of the anchor network by taking the stochastic gradient descent steps to minimize the auxiliary objective:
\begin{equation}
    \phi^{k+1} \leftarrow \arg\min_\phi   \mathcal{L}_{\text{rec}}(\mathcal{F}_\phi (X))   ,
\end{equation}
which can be viewed as approximating ${\phi}^{k+1}$ in \eqref{eq:mainloss} by taking a few gradient steps on the reconstruction loss function $\mathcal{L}_{\text{rec}}$ at each iteration and provides a progressively refined feature prior, then updating $\theta$ with respect to both task loss and alignment to the current anchor. 

\textbf{Step 2: Update weights $\theta$ of CSAG-GNN.} Given the updated frozen anchor network $\phi$, we update the CSAG-GNN parameters $\theta$ by minimizing the task loss regularized by the alignment constraint, which utilizes the fixed-weight anchor network $\mathcal{F}_{\phi}$. Therefore, the $\arg\min$ procedure can be formulated as
\begin{equation}
\label{eq:solvetheta}
\theta^{k+1} \leftarrow \arg\min_\theta \mathcal{L}_{\text{task}}( \Psi(  A, X ;\theta) , Y)    +  \frac{\rho}{2} \| \mathcal{H}_\theta^{\bar{l}} - \mathcal{F}_{\phi^{k+1}} \|^2   ,
\end{equation}
where $\mathcal{L}_{\text{task}}$ denotes the cross-entropy classification loss. Additionally, \eqref{eq:solvetheta} can be solved by the standard gradient descent algorithm.

\section{Theorem Analysis}
\label{appendix_list}

In this section, we analyze the convergence theorem of Algorithm~\ref{algo:csag-code}. For theoretical tractability and to match practical implementations, we implement the inner \(\arg\min\) in Step~2 of Algorithm~\ref{algo:csag-code} with a single gradient descent step with step size \(\beta>0\). This yields the following update rule:
\begin{align}
\phi^{k+1} &\leftarrow \phi^{k} - \eta \nabla_{\phi}\mathcal{L}_{\mathrm{rec}}(\phi^{k}), \label{eq:phi_update} \\
\theta^{k+1} &\leftarrow \theta^{k} - \beta \nabla_{\theta}\mathcal{J}^{k}(\theta^{k}), \label{eq:theta_update}
\end{align}
where
\begin{equation}
\label{eq:Jk}
\mathcal{J}^{k}(\theta) = \mathcal{L}_{\mathrm{task}}(\theta) + \frac{\rho}{2} \left\|
\mathcal{H}_{\theta}^{\bar{l}} - \mathcal{F}_{\phi^{k+1}} \right\|^{2}.
\end{equation}

Since \eqref{eq:phi_update} is independent of \(\theta\), while \eqref{eq:theta_update}
depends on \(\phi^{k+1}\). Thus the scheme has a one-way coupling \(\phi\to\theta\),
and the objective \(\mathcal{J}^{k}\) changes over iterations because the alignment anchor \(\mathcal{F}_{\phi^{k+1}}\) evolves with \(k\).

To analyze the property of the theorem of the proposed DSAL method, we impose the following assumptions.

\begin{assumption}[Smoothness and lower boundedness of reconstruction loss]
\label{assum:rec}
The reconstruction loss \(\mathcal{L}_{\mathrm{rec}}(\phi)\) is lower bounded, and
\(\nabla_{\phi}\mathcal{L}_{\mathrm{rec}}\) is \(L_{\phi}\)-Lipschitz continuous.
\end{assumption}

\begin{assumption}[Lipschitz continuity of the anchor mapping]
\label{assum:anchor_lipschitz}
There exists a constant \(L_{F}>0\) such that for all \(\phi_1,\phi_2\) in a bounded parameter
region,
\[
\left\|\mathcal{F}_{\phi_1}-\mathcal{F}_{\phi_2}\right\|
\leq
L_{F}\left\|\phi_1-\phi_2\right\|.
\]
\end{assumption}

\begin{assumption}[Bounded Jacobian of the GNN representation]
\label{assum:jacobian}
Let \(V\mapsto \sum_{i,j} V_{ij}\nabla_{\theta}[\mathcal{H}_{\theta}^{\bar{l}}]_{ij}\)
be the linear map induced by the Jacobian of the GNN representation. There exists a constant \(B_{H}>0\) such that, for all \(\theta\) in a bounded parameter region,
\[
\sup_{\|V\|\leq 1} \left\|
\sum_{i,j} V_{ij}\nabla_{\theta}[\mathcal{H}_{\theta}^{\bar{l}}]_{ij}
\right\|
\leq
B_{H}.
\]
\end{assumption}

\begin{assumption}[Smoothness of the limit structural objective]
\label{assum:limit_theta}
Let \(\phi^{*}\) be a stationary point of \(\mathcal{L}_{\mathrm{rec}}\), and according to \eqref{eq:Jk}, we define
\[
\mathcal{J}^{*}(\theta) = \mathcal{L}_{\mathrm{task}}(\theta) + \frac{\rho}{2} \left\| \mathcal{H}_{\theta}^{\bar{l}} - \mathcal{F}_{\phi^{*}} \right\|^{2}.
\]
Then \(\mathcal{J}^{*}\) is lower bounded and \(\nabla_{\theta}\mathcal{J}^{*}\) is \(L_{\theta}\)-Lipschitz continuous.
\end{assumption}

\begin{assumption}[Convergence of the anchor network parameters]
\label{assum:phi_convergence}
The sequence \(\{\phi^{k}\}\) converges to a stationary point \(\phi^{*}\), i.e.,
\[
\lim_{k\to\infty}\phi^{k}=\phi^{*},
\qquad
\nabla_{\phi}\mathcal{L}_{\mathrm{rec}}(\phi^{*})=0.
\]

This property can be guaranteed under standard sufficient decrease and relative error conditions together with the Kurdyka–\L ojasiewicz inequality.
\end{assumption}

\subsection{Convergence of the Anchor Network $\phi$}

The first step of the cyclic scheme is a standard gradient descent iteration on \(\mathcal{L}_{\mathrm{rec}}\). Hence, its convergence follows from the classical smooth optimization arguments.

\begin{lemma}[Descent of the reconstruction loss]
\label{lem:rec_descent}
Under Assumption \ref{assum:rec} and step size \(0<\eta<2/L_{\phi}\), the iterates satisfy
\[
\mathcal{L}_{\mathrm{rec}}(\phi^{k+1}) \leq \mathcal{L}_{\mathrm{rec}}(\phi^{k}) - C_{\phi} \left\|\nabla_{\phi}\mathcal{L}_{\mathrm{rec}}(\phi^{k})\right\|^{2},
\]
where
\(C_{\phi}=\eta\left(1-\frac{L_{\phi}\eta}{2}\right)>0\).
\end{lemma}

\begin{proof}
By Assumption \ref{assum:rec}, \(\nabla_{\phi}\mathcal{L}_{\mathrm{rec}}\) is
\(L_{\phi}\)-Lipschitz. Thus the descent lemma \cite{beck2017first} gives
\begin{align*}
    &\mathcal{L}_{\mathrm{rec}}(\phi^{k+1}) \\
    \leq& \mathcal{L}_{\mathrm{rec}}(\phi^{k}) + \left\langle \nabla_{\phi}\mathcal{L}_{\mathrm{rec}}(\phi^{k}), \phi^{k+1}-\phi^{k} \right\rangle  + \frac{L_{\phi}}{2} \left\|\phi^{k+1}-\phi^{k}\right\|^{2}.
\end{align*}
Substituting \eqref{eq:phi_update} and let step size \(0<\eta<2/L_{\phi}\), simplifying yields the above inequality.
\end{proof}

\begin{lemma}[Stationarity of the anchor network]
\label{lem:rec_stationary}
Under Assumptions \ref{assum:rec} and \ref{assum:phi_convergence}, we have
\[
\lim_{k\to\infty}
\left\|\nabla_{\phi}\mathcal{L}_{\mathrm{rec}}(\phi^{k})\right\|
=0.
\]
\end{lemma}

\begin{proof}
From Lemma \ref{lem:rec_descent}, summing from \(k=0\) to \(K\) gives
\[
C_{\phi}
\sum_{k=0}^{K}
\left\|\nabla_{\phi}\mathcal{L}_{\mathrm{rec}}(\phi^{k})\right\|^{2}
\leq
\mathcal{L}_{\mathrm{rec}}(\phi^{0}) - \mathcal{L}_{\mathrm{rec}}(\phi^{K+1}).
\]
Since \(\mathcal{L}_{\mathrm{rec}}\) is lower bounded, the right-hand side is uniformly bounded.
Thus \(\sum_{k=0}^{\infty}\|\nabla_{\phi}\mathcal{L}_{\mathrm{rec}}(\phi^{k})\|^{2}\) converges, which yields
\[
\lim_{k\to\infty}
\left\|\nabla_{\phi}\mathcal{L}_{\mathrm{rec}}(\phi^{k})\right\|
=0.
\]
\end{proof}

\subsection{Vanishing Gradient Drift of the Structural Subproblem}

We now analyze how the evolving anchor affects the structural update. The key feature of the proposed cyclic scheme is that the structural subproblem \(\mathcal{J}^{k}\) is not fixed during training. We now quantify the gap between the gradient of \(\mathcal{J}^{k}\) and the gradient of the limit objective \(\mathcal{J}^{*}\). For the sake of analysis, the gradient drift vector is defined by
\begin{equation}
\label{eq:drift}
d_k  = \nabla_{\theta}\mathcal{J}^{k}(\theta^{k}) - \nabla_{\theta}\mathcal{J}^{*}(\theta^{k}).
\end{equation}
Therefore, during the iteration procedure, we have the following theorem result:
\begin{lemma}[Vanishing gradient drift]
\label{lem:drift}
Under Assumptions \ref{assum:anchor_lipschitz}, \ref{assum:jacobian}, and
\ref{assum:phi_convergence}, we have
\[
\|d_k\|
\leq
\rho B_{H} L_{F}
\left\|\phi^{k+1}-\phi^{*}\right\|,
\]
and therefore 
\[
\lim_{k\to\infty}\|d_k\|=0.
\]
\end{lemma}
\begin{proof}
By definition of \(\mathcal{J}^{k}\) in \eqref{eq:Jk} and \(\mathcal{J}^{*}\) in Assumption \ref{assum:limit_theta}, we have
\[
d_k = \rho \sum_{i,j} \left( [\mathcal{F}_{\phi^{*}}]_{ij} - [\mathcal{F}_{\phi^{k+1}}]_{ij} \right)
\nabla_{\theta}
[\mathcal{H}_{\theta^{k}}^{\bar{l}}]_{ij}.
\]
Using Assumption \ref{assum:jacobian} and the Cauchy--Schwarz inequality, we obtain
\[
\|d_k\|
\leq
\rho B_{H} \left\| \mathcal{F}_{\phi^{*}} - \mathcal{F}_{\phi^{k+1}} \right\|.
\]
By Assumption \ref{assum:anchor_lipschitz}, we have
\[
\left\| \mathcal{F}_{\phi^{*}} - \mathcal{F}_{\phi^{k+1}} \right\|
\leq
L_{F} \left\|\phi^{*}-\phi^{k+1}\right\|.
\]
Thus
\[
\|d_k\| \leq \rho B_{H} L_{F} \left\|\phi^{k+1}-\phi^{*}\right\|.
\]
Under Assumption \ref{assum:phi_convergence}, when $k \rightarrow \infty$, \(\phi^{k}\to\phi^{*}\), so the right-hand side tends to zero. Therefore, we have
\[
\lim_{k\to\infty}\|d_k\|=0.
\]
\end{proof}

Lemma \ref{lem:drift} shows that the structural update can be viewed as gradient descent on the limit objective \(\mathcal{J}^{*}\) with a vanishing perturbation \(d_k\):
\[
\theta^{k+1} = \theta^{k} - \beta
\left(
\nabla_{\theta}\mathcal{J}^{*}(\theta^{k}) + d_k
\right),
\]
where $d_k  = \nabla_{\theta}\mathcal{J}^{k}(\theta^{k}) - \nabla_{\theta}\mathcal{J}^{*}(\theta^{k}).$


In the subsequent content, we now provide a general convergence lemma for this class of perturbed gradient descent updates.

\begin{lemma}[Convergence under vanishing perturbations]
\label{lem:theta_convergence}
Under Assumptions \ref{assum:limit_theta} and \ref{assum:phi_convergence}.
Let \(d_k\) be defined as in \eqref{eq:drift}. If the structural parameters are updated by
\[
\theta^{k+1} = \theta^{k} - \beta \left( \nabla_{\theta}\mathcal{J}^{*}(\theta^{k})  +  d_k \right),
\]
where \(0<\beta \leq \frac{1}{4L_{\theta}}\). Then
\[
\lim_{k\to\infty} \left\| \nabla_{\theta}\mathcal{J}^{*}(\theta^{k}) \right\| = 0.
\]
\end{lemma}

\begin{proof}
Let \(g_k = \nabla_{\theta}\mathcal{J}^{*}(\theta^{k})\), from the update rule,
\begin{equation}
\label{eq:update-theta}
\theta^{k+1} - \theta^{k}  = -\beta \left( g_k + d_k \right).
\end{equation}
Since \(\nabla_{\theta}\mathcal{J}^{*}\) is \(L_{\theta}\)-Lipschitz continuous, the descent lemma gives
\[
\mathcal{J}^{*}(\theta^{k+1}) \leq \mathcal{J}^{*}(\theta^{k}) + \left\langle g_k, \theta^{k+1}-\theta^{k} \right\rangle + \frac{L_{\theta}}{2} \left\| \theta^{k+1}-\theta^{k} \right\|^{2}.
\]
Substituting the update rule \eqref{eq:update-theta}, we obtain
\begin{align*}
    &\mathcal{J}^{*}(\theta^{k+1}) \\
    \leq& \mathcal{J}^{*}(\theta^{k}) - \beta \|g_k\|^2 - \beta \langle g_k,d_k\rangle \\ 
    &+ \frac{L_{\theta}\beta^2}{2} \left( \|g_k\|^2 + 2\langle g_k,d_k\rangle + \|d_k\|^2 \right).
\end{align*}

Rearranging terms,
\begin{align*}
    &\mathcal{J}^{*}(\theta^{k+1}) \\
    \leq& \mathcal{J}^{*}(\theta^{k}) + \left( -\beta + \frac{L_{\theta}\beta^2}{2} \right) \|g_k\|^2 + \left( -\beta + L_{\theta}\beta^2 \right) \langle g_k,d_k\rangle \\
    &+ \frac{L_{\theta}\beta^2}{2} \|d_k\|^2.
\end{align*}
Since \(0<\beta \leq \frac{1}{4L_{\theta}}\), we have
\(-\beta + L_{\theta}\beta^2 \leq -\frac{3\beta}{4}<0\).
Therefore, we have
\begin{align*}
    &\left(-\beta + L_{\theta}\beta^2\right)\langle g_k,d_k\rangle  \\
    \leq & \left(\beta - L_{\theta}\beta^2\right) \left|\langle g_k,d_k\rangle\right|  \\
    \leq & \frac{\beta - L_{\theta}\beta^2}{2} \left( \|g_k\|^2 + \|d_k\|^2 \right).
\end{align*}
Substituting this bound yields
\[
\mathcal{J}^{*}(\theta^{k+1})
\leq
\mathcal{J}^{*}(\theta^{k}) - \frac{\beta}{2} \|g_k\|^2 + \frac{\beta}{2} \|d_k\|^2.
\]
Now fix any \(\delta>0\). Since \(d_k\to 0\) by Lemma~\ref{lem:drift}, there exists \(K\) such that for all \(k\geq K\),
\[
\frac{\beta}{2}\|d_k\|^2
\leq \frac{\beta}{4}\delta^2.
\]

Assume, for contradiction, that \(\limsup_{k\to\infty}\|g_k\|_2 > 0\). Then there exist infinitely many indices \(k_j\) such that
\(\|g_{k_j}\|_2 \geq \delta\). For such \(k_j \geq K\), we have
\[
\mathcal{J}^{*}(\theta^{k_j+1})
\leq
\mathcal{J}^{*}(\theta^{k_j}) - \frac{\beta}{2}\delta^2 + \frac{\beta}{4}\delta^2
=
\mathcal{J}^{*}(\theta^{k_j}) - \frac{\beta}{4}\delta^2.
\]
Thus \(\mathcal{J}^{*}(\theta^{k_j})\) decreases by at least a positive constant infinitely often, contradicting the lower boundedness of \(\mathcal{J}^{*}\). Hence, we have
\[
\lim_{k\to\infty} \left\| \nabla_{\theta}\mathcal{J}^{*}(\theta^{k}) \right\|  =0.
\]
\end{proof}


Based on the Lemma analysis above, we have the following theoretical results:

\begin{definition}[Self-consistent stationary point]
\label{def:stationary}
A pair \((\theta^{*},\phi^{*})\) is called a self-consistent stationary point of the one-way coupled problem if it satisfies
\[
\nabla_{\phi}\mathcal{L}_{\mathrm{rec}}(\phi^{*}) = 0,
\]
and
\[
\nabla_{\theta}
\left(
\mathcal{L}_{\mathrm{task}}(\theta^{*}) + \frac{\rho}{2} \left\| \mathcal{H}_{\theta^{*}}^{\bar{l}} - \mathcal{F}_{\phi^{*}} \right\|^{2} \right) = 0.
\]
\end{definition}

\begin{assumption}[Bounded iterates]
\label{assum:bounded}
The sequences \(\{\phi^{k}\}\) and \(\{\theta^{k}\}\) generated by Algorithm~\ref{algo:csag-code} remain in compact level sets.
\end{assumption}

\begin{theorem}[Self-consistent stationary convergence]
\label{thm:main}
Under Assumptions \ref{assum:rec}--\ref{assum:bounded}, and let the step sizes
satisfy \(0<\eta < 2/L_{\phi}\) and \(0<\beta \leq 1/(4L_{\theta})\).
Then every limit point \((\theta^{*},\phi^{*})\) of the sequence generated by
Algorithm~\ref{algo:csag-code} is a self-consistent stationary point in the sense of
Definition~\ref{def:stationary}.
\end{theorem}

\begin{proof}
By Lemma~\ref{lem:rec_stationary}, the anchor network parameters satisfy
\[
\lim_{k\to\infty}  \left\|  \nabla_{\phi}\mathcal{L}_{\mathrm{rec}}(\phi^{k})  \right\|  =0.
\]
Moreover, by Lemma~\ref{lem:theta_convergence}, the structural parameters satisfy
\[
\lim_{k\to\infty} \left\| \nabla_{\theta}\mathcal{J}^{*}(\theta^{k}) \right\| =0,
\]
where
\[
\mathcal{J}^{*}(\theta)
= 
\mathcal{L}_{\mathrm{task}}(\theta) + \frac{\rho}{2} \left\| \mathcal{H}_{\theta}^{\bar{l}} - \mathcal{F}_{\phi^{*}} \right\|^{2}.
\]

Since \(\{\phi^{k}\}\) and \(\{\theta^{k}\}\) remain in compact sets by Assumption~\ref{assum:bounded}, the sequence has at least one limit point. Let \((\theta^{*},\phi^{*})\) be any such limit point, and let \(\{(\theta^{k_j},\phi^{k_j})\}\) be a subsequence converging to it. By continuity of the gradients and the two limits above, we obtain
\[
\nabla_{\phi}\mathcal{L}_{\mathrm{rec}}(\phi^{*}) = 0,
\]
and
\[
\nabla_{\theta}
\left(
\mathcal{L}_{\mathrm{task}}(\theta^{*})  + \frac{\rho}{2} \left\| \mathcal{H}_{\theta^{*}}^{\bar{l}} - \mathcal{F}_{\phi^{*}} \right\|^{2}
\right) =0.
\]
Thus \((\theta^{*},\phi^{*})\) satisfies the conditions of
Definition~\ref{def:stationary}.
\end{proof}

\begin{remark}
Although the stationarity of $\phi^k$ and the stationarity of $\theta^k$ are established separately, Theorem \ref{thm:main} additionally guarantees that their limit points are mutually consistent: the structural stationary point is computed with respect to the limiting anchor $\phi^*$, rather than an intermediate iterate.
\end{remark}


\section{Numerical Experimental}
\label{sec:Experiments}

\subsection{Dataset and Experimental Setup}
In our work, we evaluate our method on six widely used node classification benchmarks: Cora, CiteSeer, CS, Wisconsin, Texas, and Cornell. They cover both homophilous and heterophilous settings; these datasets enable a systematic evaluation across diverse structural regimes.

Cora and CiteSeer are citation networks where nodes denote publications and edges represent citations. Node features are sparse bag-of-words vectors, and labels correspond to research areas. Their strong homophily makes them standard benchmarks for evaluating neighborhood aggregation.

CS is a co-authorship graph from the Microsoft Academic Graph, with authors as nodes and collaborations as edges. Its larger scale and higher density provide a more challenging testbed for representation learning.

Wisconsin, Texas, and Cornell are heterophilous WebKB graphs in which nodes represent web pages and edges correspond to hyperlinks. The frequent connections between dissimilar nodes make them particularly challenging for conventional GNNs.

During the training phase, all learnable parameters are optimized by the Adam optimizer with an initial learning rate of 0.1. We fix the hidden representation dimension at 64 across all evaluations. Furthermore, for the specific hyperparameters introduced in our framework, the constraint coefficient $\rho$ is set to $1 \times 10^{-4}$, and the update interval $\tau$ is set to 1 by default. All experiments are implemented in Python 3.8 using PyTorch 1.9~\cite{paszke2017automatic} and conducted on an Ubuntu 20.04.6 workstation equipped with a Hygon C86 7381 64-core CPU and 64 GB of memory.

\subsection{Numerical Results}
    
\begin{table*}[thb]\centering
    \caption{Comparison experiment results for graph dataset: Node Classification Accuracy (\%). The best and second-best results are in \textbf{bold} and \underline{underlined}, respectively.}
    \label{tab:clsnode_exp1}
    \resizebox{1\textwidth}{!}
    {
    \begin{tabular}{*{8}{c}}
        \toprule
        \multirow{2}*{Methods} &  \multicolumn{3}{c}{Homophily}  &  \multicolumn{3}{c}{Heterophily}  & \multirow{2}*{Average} \\
        \cmidrule(lr){2-4}  \cmidrule(lr){5-7} \cmidrule(lr){7-7}
        ~ &  Cora \cite{yang2016revisiting}  & Citeseer \cite{yang2016revisiting} & CS \cite{yang2016revisiting}  &  Wisconsin \cite{Pei2020Geom-GCN} &  Texas \cite{Pei2020Geom-GCN} &  Cornell \cite{Pei2020Geom-GCN}  \\
        \midrule
        GraphSAGE \cite{hamilton2017inductive}  &  0.636$\pm$0.150 & 0.344$\pm$0.095   &  0.732$\pm$0.154 &  0.300$\pm$0.170  &  \underline{0.464$\pm$0.213}  &  0.227$\pm$0.087   &  0.451 \\
        
        GCN \cite{kipf2017semi}  &  0.724$\pm$0.139  & 0.589$\pm$0.158   &   0.817$\pm$0.073 & 0.294$\pm$0.171  &  0.327$\pm$0.246    & 0.289$\pm$0.122   &  0.507  \\

        APPNP \cite{KlicperaBG19} &   0.539$\pm$0.170  & \underline{0.621$\pm$0.108}  &  0.610$\pm$0.107  &  0.323$\pm$0.172   & 0.459$\pm$0.199  &     0.221$\pm$0.087   & 0.462   \\

        GIN \cite{xu2018how} &    0.291$\pm$0.079  &  0.162$\pm$0.073  &  0.228$\pm$0.016   & 0.245$\pm$0.150   &  0.245$\pm$0.158 & 0.270$\pm$0.109    &  0.240  \\

        ChebNet \cite{defferrard2016convolutional} &  0.614$\pm$0.149 &  0.312$\pm$0.116  &  \textbf{0.889}$\pm$\textbf{0.046}  &  \underline{0.462$\pm$0.153} &  0.418$\pm$0.273    & \underline{0.375$\pm$0.147}    & 0.512    \\

        GAT \cite{velickovic2018graph} &   \underline{0.770$\pm$0.037}  & 0.608$\pm$0.080   &   0.830$\pm$0.072  & 0.258$\pm$0.184   & 0.429$\pm$0.218   & 0.248$\pm$0.071    & 0.524  \\
        DSAL &  \textbf{0.794}$\pm$\textbf{0.023}    &   \textbf{0.678}$\pm$\textbf{0.040}  &   \underline{0.854$\pm$0.022}  &   \textbf{0.626}$\pm$\textbf{0.102}  &   \textbf{0.632}$\pm$\textbf{0.049} &   \textbf{0.541}$\pm$\textbf{0.060}   &  \textbf{0.688}  \\
        \bottomrule
    \end{tabular}
    }
\end{table*}

Table \ref{tab:clsnode_exp1} compares node classification accuracy on the six benchmark datasets. The evaluation is conducted on six publicly available benchmark datasets, including three homophilous graphs (Cora, Citeseer, and CS) and three heterophilous graphs (Wisconsin, Texas, and Cornell). Performance is measured in terms of the mean accuracy and standard deviation across multiple runs. The proposed DSAL achieved the highest average accuracy in the six datasets, outperforming all baseline models.

On homophilous datasets, DSAL shows balanced performance. Specifically, it achieves the highest accuracy on Cora and Citeseer, with scores of 0.794 and 0.678, respectively. On the CS dataset, DSAL obtains the second best performance, achieving an accuracy of 0.854, slightly below ChebNet, which reaches 0.889. On these datasets, DSAL combines the benefits of the two aggregation branches and remains competitive with GCN and GAT.

On heterophilous graphs, traditional GNN models such as GCN, GraphSAGE, and GIN often suffer from performance degradation when dealing with highly heterophilous graphs. Their performance is generally lower on these datasets, where neighborhood aggregation can mix information from nodes with different labels. As a result, their performance on the Wisconsin, Texas, and Cornell datasets is relatively limited, with accuracy generally ranging between 0.2 and 0.4. In contrast, DSAL achieves better improvements on these heterophilous datasets. Specifically, it reaches accuracies of 0.626 on Wisconsin and 0.541 on Cornell, corresponding to absolute gains of 16.4 and 16.6 percentage points over the second-best model, ChebNet. These results indicate that DSAL can better adapt to heterophilous neighborhood structures than the baseline methods.

\subsection{Ablation Analysis}

To validate the effectiveness of the core components within the DSAL model and the mechanisms of their interaction, we conducted extensive ablation studies on six graph datasets. Table \ref{tab:abstudy_exp} presents a performance comparison between the full model and various variants on node classification tasks. The experiments primarily evaluated the individual contributions of the alignment loss function ($\mathcal{L}_{\text{align}}$) and the gated fusion mechanism (Gate).

\begin{table*}[h]\centering
        \caption{Ablation experimental results for the graph dataset.}
    \label{tab:abstudy_exp}
    \resizebox{1\textwidth}{!}
    {
    \begin{tabular}{*{8}{c}}
        \toprule
        \multirow{2}*{ Method} &  \multicolumn{3}{c}{Homophily}  &  \multicolumn{3}{c}{Heterophily}  & \multirow{2}*{Average} \\
        \cmidrule(lr){2-4}  \cmidrule(lr){5-7} \cmidrule(lr){7-7}
        ~ &  Cora & Citeseer & CS  &  Wisconsin &  Texas &  Cornell  \\
        \midrule
        GCN   &  0.724$\pm$0.139  & 0.589$\pm$0.158   &   0.817$\pm$0.073 & 0.294$\pm$0.171  &  0.327$\pm$0.246    & 0.289$\pm$0.122   &  0.507    \\
        + $\mathcal{L}_{\text{align}}$ (w/o GraphSAGE)   &   0.792$\pm$0.049   &   0.661$\pm$0.053   &   0.871$\pm$0.055      &     0.231$\pm$0.167 &  0.289$\pm$0.209  &   0.308$\pm$0.129   & 0.525\\

        \textbf{Improv.  (\%)} & +9.39 & +12.22 & +6.61 & -21.43 & -11.62 & +6.57 & +3.68   \\
        
        \midrule
        GraphSAGE  &   0.636$\pm$0.150 & 0.344$\pm$0.095   &  0.732$\pm$0.154 &  0.300$\pm$0.170  &  0.464$\pm$0.213  &  0.227$\pm$0.087   &  0.451    \\
        + $\mathcal{L}_{\text{align}}$ (w/o GCN)   &   0.669$\pm$0.145  &   0.356$\pm$0.118  &  0.727$\pm$0.135 &        0.331$\pm$0.177   &   0.527$\pm$0.187  &   0.260$\pm$0.091  &  0.478 \\
        \textbf{Improv. (\%)} &  +5.19  &  +3.49  &  -0.68  &  +10.33  &  +13.58  &  +14.54  &  +6.18  \\  
        \midrule
        w/o $\mathcal{L}_{\text{align}}$ &   0.422$\pm$0.154  &   0.299$\pm$0.135  &  0.349$\pm$0.103  &   0.176$\pm$0.159 &  0.318$\pm$0.217  &  0.356$\pm$0.124  &    0.320 \\
        
        w/o Gate &  0.681$\pm$0.140   &   0.434$\pm$0.160 &   0.867$\pm$0.017  &       0.402$\pm$0.154  &   0.441$\pm$0.214  &   0.284$\pm$0.142  &  0.518 \\
        \textbf{DSAL}  &   0.794$\pm$0.023    &   0.678$\pm$0.040  &   0.854$\pm$0.022  &   0.626$\pm$0.102  &   0.632$\pm$0.049 &   0.541$\pm$0.060  &   0.688  \\
        \bottomrule
    \end{tabular}
    }
\end{table*}

\textbf{(a) Contribution of the Alignment Loss Function $\mathcal{L}_{\text{align}}$.} The experiments first investigated the role of $\mathcal{L}_{\text{align}}$ within a single graph encoder branch in Figure \ref{fig:csag-gnn}. The results demonstrate that the effect of $\mathcal{L}_{align}$ differs between the GCN and GraphSAGE branches.

For the GCN branch (w/o GraphSAGE), introducing $\mathcal{L}_{\text{align}}$ improves performance on homophilic graphs, including gains of 9.39\% on Cora and 12.22\% on Citeseer. However, it also increases the sensitivity of GCN to the homophily assumption, leading to performance drops on heterophilous datasets such as Wisconsin and Texas, with decreases of 21.43\% and 11.62\%, respectively.

For the GraphSAGE branch (without GCN), adding $\mathcal{L}_{\text{align}}$ improves performance on heterophilous graphs, with an average gain of more than 12\% on the three corresponding datasets. In contrast, the improvements on homophilic graphs are limited, and a slight decrease is observed on the CS dataset, with a reduction of 0.68\%. When $\mathcal{L}_{\text{align}}$ is removed entirely, the dual-branch model shows a substantial drop in average accuracy, decreasing to 0.320. This indicates that $\mathcal{L}_{\text{align}}$ plays an important role in aligning representations learned by the two encoders within a shared feature space.

\textbf{(b) Effect of the Gating Mechanism.} The ``w/o Gate'' variant removes the dynamic gating mechanism and replaces it with a standard feature fusion strategy. As shown in Table \ref{tab:abstudy_exp}, this leads to a reduction in average accuracy from 0.688 to 0.518. This suggests that simple operations such as concatenation or summation are not sufficient to effectively combine representations from the two branches in CSAG-Layer. In particular, the GCN and GraphSAGE branches capture different aspects of the graph structure, and direct fusion may introduce interference from less reliable representations, especially under heterophily settings. By introducing the gating mechanism, DSAL learns adaptive weights for features from different branches conditioned on local node context, allowing a more flexible integration of information from both homophilous and heterophilous neighborhoods.

\textbf{(c) Stability of the Model Performance.} Beyond the gains in accuracy, the ablation studies also examine the stability of the proposed model. The results of standard deviation across runs show that DSAL generally has lower variance on most datasets, such as 0.023 on Cora and 0.049 on Texas. In comparison, removing key components leads to less stable performance, particularly for the variant without $\mathcal{L}_{\text{align}}$, which shows higher variance on Texas. These results suggest that the combination of the proposed modules helps improve not only performance but also consistency across different data splits and random initializations.

\subsection{Analysis of Topological Perturbation Experiments}
\label{sec:edgePeerturb}

To evaluate model robustness under structural noise, we conduct edge perturbation experiments on both a homophilic dataset (Citeseer) and a heterophilous dataset (Cornell). The graph topology is modified by randomly adding or removing 10\%, 30\%, and 50\% of edges. Tables \ref{tab:edgeadd_exp} and \ref{tab:edgeremove_exp} present the node classification results of DSAL and baseline models under edge addition and edge removal settings, respectively. Specifically, we utilized standard training procedures during the training phase, while in the testing phase, a proportion of edges is randomly added or removed to assess model robustness under topology perturbations.

\subsubsection*{(1) Edge Addition Analysis}

\begin{table*}[h]\centering
        \caption{Edge addition perturbation experimental results for the graph dataset on the three perturbation ratios.} 
    \label{tab:edgeadd_exp}
    \resizebox{1\textwidth}{!}
    {
    \begin{tabular}{*{7}{c}}
        \toprule
        \multirow{2}*{Methods} &  \multicolumn{3}{c}{Citeseer}  &  \multicolumn{3}{c}{Cornell}   \\
        \cmidrule(lr){2-4}  \cmidrule(lr){5-7} 
        ~ &  10\% & 30\% & 50\%  &  10\% &  30\% &  50\%  \\
        \midrule
        APPNP &    0.521$\pm$0.100  &  0.436$\pm$0.083 &  0.356$\pm$0.069       &    0.224$\pm$0.086    &   0.240$\pm$0.092 &    0.251$\pm$0.109    \\
        ChebNet &  0.288$\pm$0.089  &  0.293$\pm$0.059   &   0.289$\pm$0.075    &     0.427$\pm$0.172   &    0.410$\pm$0.161    &  0.386$\pm$0.136  \\
        GAT &   0.568$\pm$0.057    &    0.500$\pm$0.048   &   0.474$\pm$0.048    &     0.267$\pm$0.097   &    0.291$\pm$0.142   &    0.237$\pm$0.089  \\
        GCN  &  0.566$\pm$0.147 &     0.524$\pm$0.124   &   0.484$\pm$0.115   &      0.300$\pm$0.129   &0.291$\pm$0.130   &  0.286$\pm$0.131  \\
        GIN &   0.170$\pm$0.081  &    0.158$\pm$0.066  &     0.148$\pm$0.056    &   0.291$\pm$0.118  &   0.294$\pm$0.121  &  0.289$\pm$0.118   \\
        GraphSAGE  &  0.313$\pm$0.097   &    0.300$\pm$0.082  &     0.307$\pm$0.069  &   0.251$\pm$0.051   &   0.316$\pm$0.116  &   0.273$\pm$0.103 \\
        
        DSAL &    \textbf{0.629}$\pm$\textbf{0.036}  &      \textbf{0.601}$\pm$\textbf{0.034} &   \textbf{0.573}$\pm$\textbf{0.036}       &    \textbf{0.459}$\pm$\textbf{0.092}  &   \textbf{0.464}$\pm$\textbf{0.053}   &   \textbf{0.451}$\pm$\textbf{0.065}  \\
        
        \bottomrule
    \end{tabular}
    }
\end{table*}

In this experiment, adding random edges introduces spurious connections in the graph, which may cause GNN models to aggregate less relevant or noisy neighbor information during message passing and further affect performance. As shown in Table \ref{tab:edgeadd_exp}, models that rely more heavily on the homophily assumption tend to experience noticeable performance drops as the perturbation ratio increases from 10\% to 50\%. For example, on the Citeseer dataset, the accuracy of APPNP decreases from 0.521 to 0.356, while on the Cornell dataset, GAT and GCN maintain relatively low performance across all perturbation levels.

In comparison, DSAL shows more stable performance under edge addition perturbations. On Citeseer, it achieves an accuracy of 0.573 even at a perturbation ratio of 50\%, outperforming the other evaluated models. On the Cornell dataset, DSAL also maintains the best results across all settings, with accuracies of 0.459, 0.464, and 0.451 under perturbation ratios of 10\%, 30\%, and 50\%, respectively. These results suggest that the proposed feature fusion and gating design can reduce the impact of noisy edges by limiting the influence of less reliable structural information during aggregation.

\subsubsection*{(2) Edge Removal Analysis}

\begin{table*}[h]\centering
        \caption{Edge removal perturbation experimental results for the graph dataset on the three perturbation ratios.} 
    \label{tab:edgeremove_exp}
    \resizebox{1\textwidth}{!}
    {
    \begin{tabular}{*{7}{c}}
        \toprule
        \multirow{2}*{Methods} &  \multicolumn{3}{c}{Citeseer}  &  \multicolumn{3}{c}{Cornell}   \\
        \cmidrule(lr){2-4}  \cmidrule(lr){5-7} 
        ~ &  10\% & 30\% & 50\%  &  10\% &  30\% &  50\%  \\
        \midrule
        APPNP &    0.567$\pm$0.138   &  0.610$\pm$0.101  &    0.572$\pm$0.140   &    0.262$\pm$0.076  &   0.262$\pm$0.076   &      0.281$\pm$0.114    \\
        ChebNet &  0.310$\pm$0.087  &   0.317$\pm$0.097  &    0.276$\pm$0.097 &      0.340$\pm$0.138  &   0.418$\pm$0.176  &    0.386$\pm$0.139         \\
        GAT &  0.608$\pm$0.078  &    0.617$\pm$0.065  &   0.595$\pm$0.058   &       0.254$\pm$0.093  &  0.294$\pm$0.110 &     0.229$\pm$0.068   \\
        GCN  &  0.585$\pm$0.162  &     0.580$\pm$0.151  &    0.570$\pm$0.149   &  0.329$\pm$0.113  &    0.302$\pm$0.130 &    0.300$\pm$0.129   \\
        GIN &  0.160$\pm$0.072   &    0.167$\pm$0.085  &   0.150$\pm$0.064  &    0.289$\pm$0.115  & 0.282$\pm$0.115  &  0.283$\pm$0.110 \\
        GraphSAGE  & 0.331$\pm$0.097  &   0.343$\pm$0.100  &    0.327$\pm$0.101   &  0.256$\pm$0.098  &   0.227$\pm$0.064   &   0.251$\pm$0.055   \\
        
        DSAL &   \textbf{0.670}$\pm$\textbf{0.037} &  \textbf{0.659}$\pm$\textbf{0.045} &   \textbf{0.650}$\pm$\textbf{0.035}  &       \textbf{0.497}$\pm$\textbf{0.082} &  \textbf{0.462}$\pm$\textbf{0.077}  &  \textbf{0.467}$\pm$\textbf{0.106}   \\
        \bottomrule
    \end{tabular}
    }
\end{table*}

Unlike edge addition, edge removal reduces graph connectivity and leads to the loss of structural information, which can hinder message passing and feature aggregation. As shown in Table \ref{tab:edgeremove_exp}, this setting has a negative effect on all baseline models. For instance, on the Citeseer dataset, when the edge removal ratio increases to 50\%, the accuracy of GCN and GAT drops to 0.570 and 0.595, respectively. On the Cornell dataset, most traditional GNNs achieve accuracy below 0.350 under the same conditions.

In comparison, DSAL shows more stable performance under edge removal. On Citeseer, it achieves an accuracy of 0.650 even when half of the edges are removed, outperforming all baseline models. On the Cornell dataset, DSAL also maintains relatively stable results, with an accuracy of 0.467 and a standard deviation of 0.106, and performs better than methods such as ChebNet under most settings. These results indicate that the proposed model is less sensitive to reduced connectivity and can better utilize the remaining structural and feature information in sparse graphs.

The results in the two tables show that under stronger topological perturbations, such as 30\% and 50\% edge modification, the variance of models like GCN and APPNP often exceeds 0.1 on the Citeseer dataset. This suggests that their performance is more sensitive to randomness in the perturbation process. In contrast, DSAL maintains a lower standard deviation. This stability is mainly related to the dual-branch architecture and the gating mechanism. When the graph structure is heavily perturbed, the gating module adjusts the contribution of each branch and reduces the influence of less reliable structural signals, relying more on relatively stable node-level information. This helps maintain more consistent predictions under both edge addition and edge removal settings.

\subsection{Hyperparameter Analysis}

\subsubsection*{(1) Effect of the different hyperparameter $\rho$ and $\tau$}

To evaluate the sensitivity of DSAL to hyperparameters, we conduct experiments on six benchmark datasets. The analysis focuses on two parameters: the constraint hyperparameter $\rho$, which controls the relaxation strength of feature–structure alignment, and the update interval $\tau$ for alternating optimization between the encoder and decoder in the feature view. The results are shown in Figure \ref{fig:sensitivity}, where the solid curves represent average classification accuracy under different settings, and the shaded regions indicate the corresponding standard deviations.

\begin{figure*}[h]
    \centering
    \includegraphics[width=1.0\linewidth]{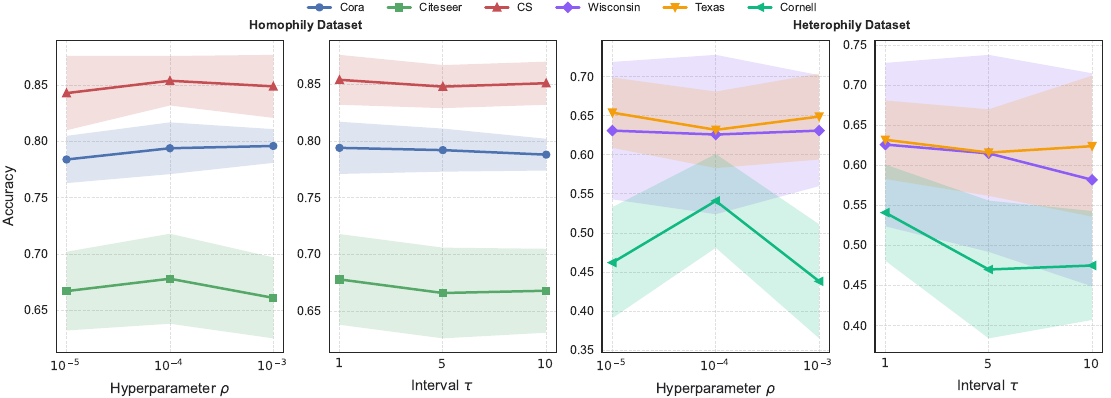}
    \caption{The sensitivity analysis of hyperparameter $\gamma$ and update interval $\tau$ on different graph datasets.}
    \label{fig:sensitivity}
\end{figure*}

The subplots corresponding to $\rho$ in Figure \ref{fig:sensitivity} show that when $\rho$ varies within the range of $[10^{-5}, 10^{-3}]$, the performance remains relatively stable on both homophilous datasets, such as Cora and CS, and heterophilous datasets, such as Texas and Wisconsin. The changes in accuracy are small, and the shaded regions indicating variance do not exhibit noticeable variation. This observation implies that the model's sensitivity to $\rho$ is relatively low for most datasets.

As shown in the subplots for $\tau$ in Figure  \ref{fig:sensitivity}, the performance curves remain relatively stable on homophilous datasets, suggesting that the update interval has limited influence in these cases. In contrast, on heterophilous datasets, especially Cornell, the average accuracy decreases as $\tau$ increases beyond 1, and the variance becomes slightly larger. In homophilous graphs, neighboring nodes tend to share similar features, so the alignment signal has a limited effect on aggregation. In heterophilous graphs, where neighbors are often less consistent, the model relies more on feature-level information to correct structural noise. A larger $\tau$ reduces the frequency of this alignment process, which may weaken this corrective effect and lead to a mild drop in performance.

\subsubsection*{(2) Effect of the Allocation Ratio $\gamma$}

In the structural design of DSAL as shown in Figure \ref{fig:csag-gnn}, we adopt a channel-split gated convolution mechanism, where a hyperparameter $\gamma \in (0,1)$ controls the split of feature channels in a dual-branch architecture. Specifically, a proportion $\gamma$ of channels is processed by the GCN branch, while the remaining $(1-\gamma)$ is processed by the GraphSAGE branch. To study the effect of this allocation, we conduct ablation experiments on six benchmark datasets with $\gamma$ set to ${0.3, 0.5, 0.8}$, and report the results in Table \ref{tab:alloc_exp}.

\begin{table*}[thb]\centering
        \caption{Allocate ratio $\gamma$ sensitive experimental results for the graph dataset.} 
    \label{tab:alloc_exp}
    \resizebox{0.9\textwidth}{!}
    {
    \begin{tabular}{*{7}{c}}
        \toprule
        \multirow{2}*{ $\gamma$} &  \multicolumn{3}{c}{Homophily}  &  \multicolumn{3}{c}{Heterophily}   \\
        \cmidrule(lr){2-4}  \cmidrule(lr){5-7} 
        ~ &  Cora & Citeseer & CS  &  Wisconsin &  Texas &  Cornell  \\
        \midrule

        0.3   &  0.794$\pm$0.023    &   0.678$\pm$0.040  &   0.854$\pm$0.022  &   0.626$\pm$0.102  &   0.632$\pm$0.049 &   0.541$\pm$0.060   \\
                
        0.5   &   0.783$\pm$0.024   &   0.680$\pm$0.014   &    0.827$\pm$0.033  &   0.561$\pm$0.101  &  0.630$\pm$0.065  & 0.468$\pm$0.079 \\

        0.8  &   0.804$\pm$0.008  &    0.662$\pm$0.042  &   0.852$\pm$0.033     &   0.509$\pm$0.096   &   0.629$\pm$0.047   &   0.475$\pm$0.078  \\
        \bottomrule
    \end{tabular}
    }
\end{table*}

When $\gamma$ is set to 0.3, more feature channels are assigned to the GraphSAGE branch. This provides additional ability for processing homophilous patterns, while the remaining GCN channels still contribute structural information through the gating mechanism. Under this setting, the model performs relatively well on heterophilous datasets, without significant performance loss on homophilous datasets. Based on these results, $\gamma = 0.3$ is used as the default setting in our experiments.

\section{Discussions}

\subsection{Alternating Update Vs  Joint Update}
\label{sec:ALTvsNormal}

The joint mode training strategy directly combines the task loss $\mathcal{L}_{\text{task}}$ with auxiliary losses such as $\mathcal{L}_{\text{align}}$ and $\mathcal{L}_{\text{rec}}$, which updates all parameters jointly in the one-step forward pass. This type of update method is frequently seen in earlier algorithms. To compare this approach with the alternating optimization strategy adopted in this work, which updates the feature view (Step 1) and structural view (Step 2) in an alternating manner, we conduct experiments on six benchmark datasets. The results are presented in Table \ref{tab:modeVs_exp}.

\begin{table*}[thb]\centering
        \caption{The experimental results of different update modes for the graph dataset.}
    \label{tab:modeVs_exp}
    \resizebox{0.9\textwidth}{!}
    {
    \begin{tabular}{*{7}{c}}
        \toprule
        \multirow{2}*{Mode} &  \multicolumn{3}{c}{Homophily}  &  \multicolumn{3}{c}{Heterophily}   \\
        \cmidrule(lr){2-4}  \cmidrule(lr){5-7} 
        ~ &  Cora & Citeseer & CS  &  Wisconsin &  Texas &  Cornell  \\
        \midrule


        Alt Mode   &    0.794$\pm$0.023    &   0.678$\pm$0.040  &   0.854$\pm$0.022  &   0.626$\pm$0.102  &   0.632$\pm$0.049 &   0.541$\pm$0.060  \\

        Joint Mode &    0.799$\pm$0.010   &    0.673$\pm$0.043  &   0.856$\pm$0.017  &       0.594$\pm$0.157   &   0.624$\pm$0.078    &  0.487$\pm$0.081\\
        \bottomrule
    \end{tabular}
    }
\end{table*}

On homophilous datasets, the two update modes show comparable performance, with the joint mode even achieving slightly higher accuracy on Cora and CS. This suggests that when node features are consistent with local graph structure, the optimization path is relatively smooth, and both joint and alternating updates tend to converge to similar solutions.

In heterophilous graphs with more complex connectivity patterns, the alternating update strategy performs more consistently than the joint mode. As shown in Table \ref{tab:modeVs_exp}, it achieves accuracies of 0.626 and 0.541 on Wisconsin and Cornell, respectively, compared with 0.594 and 0.487 for the ``Joint update''. The joint update style also shows higher variance on heterophilous datasets, such as 0.157 on Wisconsin, indicating less stable training behavior under these settings. 

These differences suggest that the performance gap under heterophily may be related to the interaction between the two optimization objectives. Specifically, in heterophilous graphs, joint optimization in a single step can cause gradients from different losses to interfere, making it difficult to balance classification and alignment objectives. In contrast, the “Alt-Mode” proposed in this work decouples the feature view and the structural view at the optimization level. This staged training strategy reduces interference between the two views, allowing each branch to be optimized more independently and converge within its own representation space. As a result, it produces more stable and discriminative representations under heterophilous structural noise.

\subsection{Overhead Analysis}

In this section, we analyze the computational overhead of DSAL in terms of model size, computational cost, and inference time. We compare DSAL with representative GNN baselines under different hidden dimensions to assess the resource consumption.

\begin{table*}[thb]\centering
        \caption{The overhead performance of different methods on the graph datasets under varying hidden dimensions.} 
    \label{tab:overhead_exp}
    \resizebox{0.8\textwidth}{!}
    {
    \begin{tabular}{*{11}{c}}
        \toprule
        \multirow{2}*{Method} &  \multicolumn{3}{c}{Params (M)}  &  \multicolumn{3}{c}{FLOPs (M)}   &  \multicolumn{3}{c}{Inference Time (ms)}  \\
        \cmidrule(lr){2-4}  \cmidrule(lr){5-7}  \cmidrule(lr){8-10} 
        ~  & 32 & 64 & 96    & 32 & 64  & 96  & 32 & 64  & 96  \\
        \toprule 
        GCN   &    0.047        &      0.096     &     0.147                              & 127.5   & 260.6  &  399.3                     &  2.742  &   2.791  &   2.852  \\

        GraphSAGE     &   0.094     &      0.192     &    0.295                                 &   255.1    &      521.3    & 798.6    &                          1.537    &   1.580  &   1.601   \\ 
  
        DSAL    &   0.082       &      0.165     &     0.253                           &      223.2   &            448.1      &  686.6   &              4.061    &    4.036 &    4.242   \\
        \bottomrule
    \end{tabular}
    }
\end{table*}

As shown in Table \ref{tab:overhead_exp}, we evaluate the overhead of the proposed DSAL alongside GCN and GraphSAGE methods on the Cora dataset. From a theoretical complexity perspective, DSAL demonstrates balanced performance. It requires fewer parameters and computational operations (FLOPs) compared to GraphSAGE under all hidden dimensions. However, we observe that the inference time of DSAL is higher than that of other baseline methods. This discrepancy between lower FLOPs and higher inference latency is likely caused by the current implementation limitations. Specifically, the channel-splitting and gated fusion mechanisms in DSAL require frequent tensor slicing, concatenation, and discontinuous memory access. Unlike the GCN and GraphSAGE, which benefit from highly fused, natively CUDA operators in standard libraries, our current implementation operates at a higher framework level, leading to increased memory-bound overhead. Future work can focus on developing customized CUDA kernels to fuse these operations, thereby aligning the empirical inference time with its theoretical efficiency.

\section{Conclusion and Future Work}

In this paper, we study the optimization of GNNs under noisy and unreliable graph topologies and propose a structure–feature constrained learning method derived from a relaxed optimization formulation. By treating the node feature representation as a relatively stable reference, we adopt a cyclic alternating update strategy. This design separates the optimization of structural and feature-related features, thereby reducing the gradient interference induced by structural noise and mitigating the representation drift observed in joint optimization. Meanwhile, the proposed CSAG-Layer introduces an adaptive gating mechanism that balances the global smoothing effect of graph convolution with the preservation of local information of neighborhood aggregation. Experiments show that, compared with conventional message-passing and joint optimization methods, our approach achieves more balanced and discriminative representations with only a modest increase in model complexity.

The current implementation has two main limitations:

\textbf{Inference efficiency:} Although DSAL has balanced theoretical FLOPs and parameter size compared with standard GNNs, its practical inference latency is higher. This is mainly due to channel-splitting and adaptive gating operations, which introduce non-contiguous memory access and increase the inference latency.

\textbf{Static channel allocation:} The channel split ratio $\gamma$ is fixed across nodes and layers, limiting the model’s ability to adaptively allocate representation capacity to heterophilous structural patterns.

Future research will focus on improving both efficiency and adaptability. From a systems perspective, we plan to design fused GPU kernels to reduce memory overhead in channel-wise operations. From an algorithmic perspective, we will investigate adaptive optimization strategies to control the allocation ratio $\gamma$ during training dynamically. In addition, extending the proposed constrained alignment framework to dynamic graphs remains an important direction for future study.

\section*{Acknowledgment}
\noindent{\bf Funding:} This research is supported by the National Natural Science Foundation of China (NSFC) grants 92473208, 12401415, the Key Program of National Natural Science of China 12331011, the 111 Project (No. D23017),  the Natural Science Foundation of Hunan Province (No. 2025JJ60009), Project of Scientific Research Fund of Hunan Provincial Science and Technology Department (No. 25B0148). 

\noindent{\bf Data Availability:} Enquiries about data/code availability should be directed to the authors.

\noindent{\bf Competing interests:} The authors have no competing interests to declare that are relevant to the content of this paper.






\bibliographystyle{IEEEtran}    
\bibliography{refs}

\end{document}